\documentclass[dvipsnames,format=sigconf, authorversion=True]{acmart}
\usepackage[utf8]{inputenc}
\usepackage{eqnarray}
\usepackage{algorithm}
\usepackage[capitalise,noabbrev]{cleveref}

\usepackage{enumitem}
\usepackage[noend]{algpseudocode}
\usepackage{amsmath}
\usepackage{xspace}
\usepackage{mathtools}
\usepackage{appendix}
\usepackage{microtype}

\usepackage{xcolor}
\usepackage{booktabs} 
\usepackage{makecell}
\usepackage{graphicx} 
\usepackage{subcaption}
\usepackage{graphicx}

\newcommand{\walkertwo}{\textsc{walker2d-2}\xspace}
\newcommand{\hopperthree}{\textsc{hopper-3}\xspace}
\newcommand{\halfcheetah}{\textsc{halfcheetah-2}\xspace}
\newcommand{\kheperaxmulti}{\textsc{kheperax-2}\xspace}
\newcommand{\antmulti}{\textsc{ant-2}\xspace}

\newcommand{\mourqd}{\textsc{mour-qd}\xspace}
\newcommand{\mourqdlong}{Multi-Objective Unstructured Repertoire for Quality-Diversity\xspace}

\newcommand{\mome}{\textsc{mome}\xspace}

\newcommand{\moauroragrid}{\textsc{mo-aurora-grid}\xspace}
\newcommand{\momesmall}{\textsc{mome-small}\xspace}
\newcommand{\momelarge}{\textsc{mome-large}\xspace}

\newcommand{\mapelites}{\textsc{map-elites}\xspace}
\newcommand{\aurora}{\textsc{aurora}\xspace}

\newcommand{\moqd}{\textsc{moqd}\xspace}
\newcommand{\uqd}{\textsc{uqd}\xspace}

\newcommand{\qd}{\textsc{qd}\xspace}
\newcommand{\mo}{\textsc{mo}\xspace}

\newcommand{\replications}{\textsc{10}\xspace}

\newcommand{\qdscore}{\textsc{qd-score}\xspace}
\newcommand{\moqdscore}{\textsc{moqd-score}\xspace}

\newcommand{\globalhypscore}{\textsc{global-hypervolume}\xspace}
\newcommand{\coverage}{\textsc{coverage}\xspace}

\title{Multi-Objective Quality-Diversity in Unstructured and Unbounded Spaces}
\author{Hannah Janmohamed}
\affiliation{%
  \institution{Imperial College London}
  \city{London}
  \country{UK}}
\email{hnj21@imperial.ac.uk}

\author{Antoine Cully}
\affiliation{%
  \institution{Imperial College London}
  \city{London}
  \country{UK}}
\email{a.cully@imperial.ac.uk}

\acmDOI{10.1145/3712256.3726394}
\acmISBN{979-8-4007-1465-8/2025/07}
\acmConference[GECCO '25]{The Genetic and Evolutionary Computation Conference 2025}{July 14--18, 2025}{Málaga, Spain}
\acmYear{2025}
\copyrightyear{2025}

\date{October 2024}

\begin{document}

\begin{abstract}

Quality-Diversity algorithms are powerful tools for discovering diverse, high-performing solutions.
Recently, Multi-Objective Quality-Diversity (\moqd) extends \qd to problems with several objectives while preserving solution diversity. 
\moqd has shown promise in fields such as robotics and materials science, where finding trade-offs between competing objectives like energy efficiency and speed, or material properties is essential.
However, existing methods in \moqd rely on tessellating the feature space into a grid structure, which prevents their application in domains where feature spaces are unknown or must be learned, such as complex biological systems or latent exploration tasks.
In this work, we introduce \mourqdlong (\mourqd), a \moqd algorithm designed for unstructured and unbounded feature spaces.
We evaluate \mourqd on five robotic tasks.
Importantly, we show that our method excels in tasks where features must be learned, paving the way for applying \moqd to unsupervised domains.
We also demonstrate that \mourqd is advantageous in domains with unbounded feature spaces, outperforming existing grid-based methods. 
Finally, we demonstrate that \mourqd is competitive with established \moqd methods on existing \moqd tasks and achieves double the \moqdscore in some environments. 
\mourqd opens up new opportunities for \moqd in domains like protein design and image generation.
\end{abstract}



\maketitle

\vspace{-10pt}
\section{Introduction}

\begin{figure}
    \centering
    \includegraphics[width=0.9\linewidth]{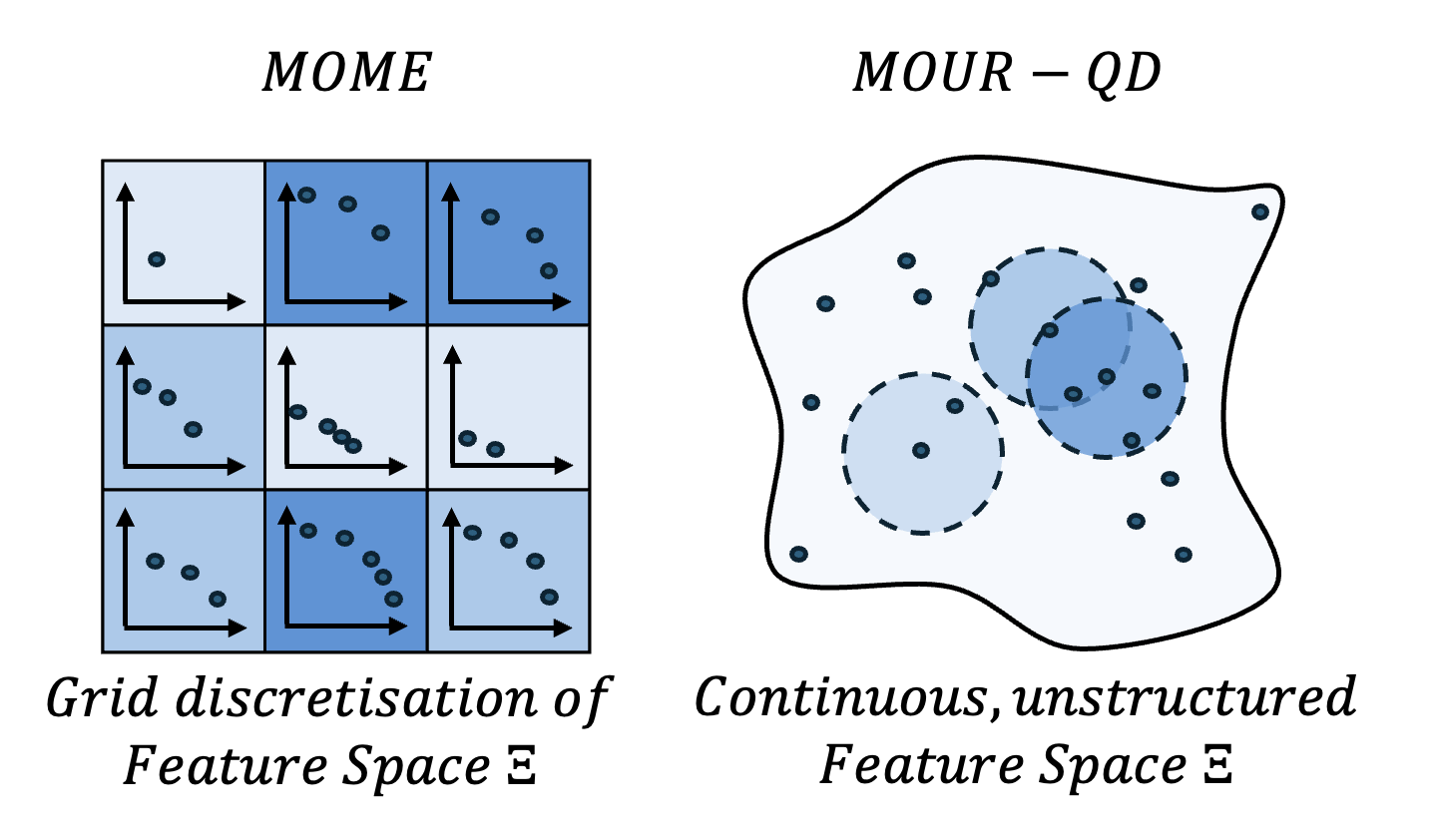}
    \caption{\textit{Left:} Existing \moqd methods store Pareto Fronts in each cell of a MAP-Elites grid. \textit{Right:} \mourqd introduces an unstructured archive, storing solutions in a continuous manner throughout the feature space. Pareto Fronts are defined locally, using a radius $2l$ around each solution.}
    \label{fig:teaser}
\end{figure}

In recent years, Quality-Diversity (\qd) algorithms \cite{mapelites, qdunifying} have gained significant attention for their ability to discover a diverse set of high-performing solutions.
Unlike traditional optimisation methods that focus on identifying a single optimal solution, \qd methods aim to uncover a collection of solutions that vary across different behaviours or \textit{features} while still performing well.
This ability to balance exploration and optimisation has proven valuable across a range of real-world, complex applications.
For example, in robotics, \qd has been applied to damage recovery by providing fallback behaviours that allow robots to adapt to hardware failures, such as broken limbs or motors \cite{nature, hbr}.
Alternatively, \qd has been used to generate diverse adversarial prompts in order to train Large Language Models, helping identify vulnerabilities and improve model safety by testing the system across a broad range of inputs \cite{rainbowteaming}.

Building on the foundations of \qd, Multi-Objective Quality-Diversity (\moqd) extends this framework to tackle problems where solutions must balance multiple competing objectives while maintaining behavioural diversity \cite{mome, mome-pgx, mome-p2c}.
This approach has proven valuable across a range of applications too.
In robotics, \moqd enables exploration of trade-offs between objectives such as energy consumption and speed, or fitness and reproducibility, while preserving a diverse repertoire of behaviours \cite{mome, mome-pgx, mome-p2c, performance-reproducibility}. 
In materials design, it facilitates the discovery of materials which achieve different trade-offs across material objectives, such as strength, durability, and conductivity, with diverse properties, providing alternative options if some materials prove impractical to synthesise \cite{moqdcsp}.
In architecture, \moqd supports the generation of building designs that balance objectives like ventilation efficiency and noise reduction, while offering multiple diverse options to suit different requirements or preferences \cite{mcx}.
\moqd plays a critical role in solving problems that require optimisation across multiple objectives while preserving diversity, making it a powerful tool for tackling complex and multi-faceted challenges.

Despite these advances, existing \moqd methods rely on discretising the feature space into a structured \mapelites grid, where each cell contains a Pareto front as illustrated in \Cref{fig:teaser}.
While effective, this approach assumes that the features of the tasks can be hand-defined and that their bounds are known beforehand.
These assumptions can be restrictive, particularly in scenarios where the defining features of solutions are unclear or difficult to specify, or when the bounds of the feature space are unknown.
For example, tasks requiring latent exploration \cite{aurora-original, aurora-csc, dqd} are often unsuitable for grid-based \qd algorithms because the shape of the latent space is unknown in advance and could dynamically change.

In single-objective \qd, unstructured archives have been used to enable automatic feature learning \cite{aurora-original, aurora-csc}, allowing solutions to be organised without predefined feature boundaries.
However, adapting this concept to the multi-objective setting is particularly challenging.
In \moqd, the grid structure provides a natural way to organise and maintain Pareto fronts within each cell.
Without the grid, it becomes unclear how to define Pareto fronts in a continuous, unstructured feature space, posing a significant barrier to extending unstructured containers to multi-objective problems.
As a result, no existing methods address these challenges in \moqd, which limits their applicability to problems requiring latent space exploration or unsupervised feature discovery.

This paper presents \mourqdlong (\mourqd), a novel \moqd algorithm that uses a multi-objective unstructured archive, as illustrated in \Cref{fig:teaser}.
Unlike existing \moqd approaches, \mourqd does not rely on predefined features or fixed feature space bounds.
Instead, it stores solutions in an unstructured and unbounded manner and defines Pareto Fronts locally around solutions.

We evaluate \mourqd across 5 continuous control robotics tasks and show that it offers three key advantages.
First, we evaluate \mourqd on existing \moqd tasks and demonstrate that it is competitive with current \moqd baselines and doubles the \moqdscore of current \moqd baselines in some environments.
This confirms that \mourqd is an effective solution for advancing existing \moqd applications.
Second, we show that \mourqd improves on current \moqd baselines in its robustness to cases where feature space bounds are unknown.
By removing the need for prior knowledge of feature space limits, \mourqd makes \moqd algorithms more flexible and adaptable to a wider range of complex, real-world problems.
Finally, we show that \mourqd enables \moqd to operate in unsupervised settings by learning features directly from data.
Crucially, this capability opens the door for applying \moqd to entirely new domains, such as protein design and latent space exploration, where manually defining features is impractical.
By addressing these challenges, \mourqd represents a meaningful step forward in the evolution of \moqd, unlocking new opportunities for exploring diversity and trade-offs in challenging optimisation problems.
All of our code was implemented using the QDax codebase \cite{qdax, qdaxrepo} which uses massive parallelisation via Jax \cite{jax2018github}.
Our code was fully containerised and is publicly available at \url{https://github.com/adaptive-intelligent-robotics/MOUR-QD}.

\section{Background and Related Works}

\subsection{Multi-Objective Optimisation}\label{sec:backgroundmoo}

\begin{figure}
    \centering
    \includegraphics[width=0.4\linewidth]{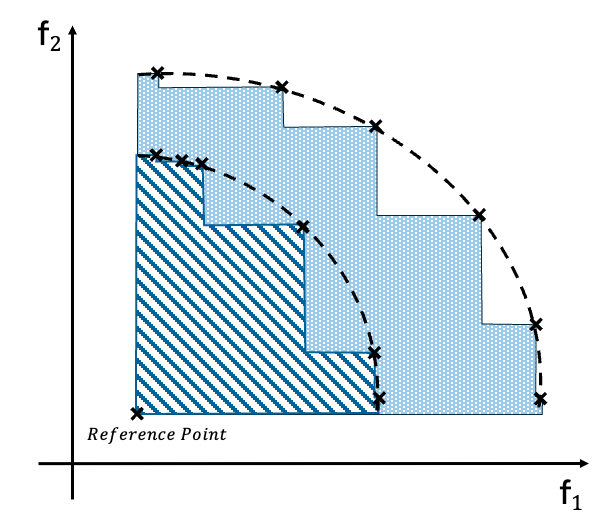}
    \caption{Two sets of solutions forming two Pareto Fronts. The hypervolume is reflected by the shaded areas in the objective space between the solutions and the front. The outer set of solutions achieve higher in both objectives and thus this front has a higher hypervolume.}
    \label{fig:pfs}
\end{figure}

Traditional optimisation algorithms involve finding a solution $\theta \in \Theta$ that maxises the performance of an objective function $f: \Theta \rightarrow \mathbb{R}$.
By contrast, Multi-Objective Optimization (\mo) extends traditional single-objective optimization by involving several objectives that must be optimized simultaneously.
Typical examples of \mo problems include optimising performance and cost in engineering design, maximising accuracy while minimising complexity in machine learning models, or balancing resource allocation between different competing demands in scheduling or logistics \cite{moea_survey, mooreview}.

Formally, in \mo, a solution $\theta \in \Theta$ is evaluated using a vector of objective functions $\mathbf{F}(\theta) = [f_1(\theta), f_2(\theta), \dots, f_m(\theta)]$, where $f_i: \Theta \to \mathbb{R}$ represents the $i$-th objective \cite{moobook, mooreview}. 
If one solution performs better than another solution across all of the objectives, we say that it \textit{dominates} it.
However, objectives are often conflicting: improving one objective, such as minimising cost, may worsen another, such as maximising performance \cite{moea_survey}.
Therefore, some solutions may not dominate one-another, but rather offer a different trade-off across the objectives.
Consequently, unlike in single-objective optimisation which aims to find a single optimal solution, \mo aims to find a set of non-dominated solutions that offer different trade-offs across competing objectives.

The solutions that represent the best trade-offs are those for which no other solution can improve one objective without degrading at least one other.
These solutions are referred to as \textit{Pareto-optimal}. 
Mathematically, a solution $\theta$ is Pareto-optimal if there exists no other solutions $\theta' \in \Theta$ such that $f_i(\theta') \geq f_i(\theta) \forall i \in \{1, \dots, m\}$, with at least one strict inequality.
The collection of all such solutions from the entire search space $\Theta$ forms the \textit{Pareto front}, denoted as $\mathcal{P}(\Theta)$.

Due to the vastness and complexity of many search spaces, it is often infeasible to find the true Pareto Front, and instead algorithms aim to find a close approximation.
The performance of \mo algorithms is often assessed using the hypervolume metric.
Given a reference point $\mathbf{r} \in \mathbb{R}^m$ that is dominated by all solutions in the objective space, the hypervolume measures the volume of the region dominated by the solutions in the Pareto front relative to $\mathbf{r}$. 
Qualitatively, as visualised in \Cref{fig:pfs} it represents the extent of the objective space “covered” by the Pareto front, where larger hypervolume values indicate better exploration and optimization of the trade-offs.
Mathematically, the hypervolume $\Xi(\mathcal{P})$ of a Pareto Front $\mathcal{P}$ is computed as:
\begin{equation}
    \Xi(\mathcal{P}) = \lambda ({\theta \in \Theta \,|\, \exists \, s \in \mathcal{P}, s \succ x \succ \mathbf{r}})
   \label{eqn:hypervolume}
\end{equation}
where $\lambda$ denotes the Lebesque measure \cite{usinghypervolumes, guidetomorl}.

\subsection{Quality-Diversity}

Quality-Diversity (\qd) algorithms \cite{mapelites, qdunifying} aim to generate a diverse collection of high-performing solutions.
The performance of a solution is measured by a \textit{fitness} function $f: \Theta \to \mathbb{R}$, while its characteristics are referred to as \textit{features} (also known as behaviour descriptors or measures in the literature \cite{qdunifying}) and are captured by a feature function $\Phi: \Theta \to \mathbb{R}^d$.

In general, \qd algorithms are initialised by randomly generating solutions and adding them to an archive $\mathcal{A}$, which acts as a repository for storing high-quality, diverse solutions.
At each iteration, a batch of solutions $\theta_1, ..., \theta_b$ are selected from the archive and subjected to genetic variation, such as mutation or crossover, to produce offspring. These offspring are then evaluated to obtain their fitnesses $f(\theta_1),..., f(\theta_b)$ and feature values $\Phi(\theta_1),..., \Phi(\theta_b)$.
The archive is updated by incorporating these offspring if they are either higher-performing than existing solutions or are novel compared to the other solutions in the archive.

The precise mechanisms for adding solutions to the archive vary depending on the type of archive being used. 
In structured archives, such as those used by \mapelites algorithms \cite{mapelites, nature, mome}, the feature space is tessellated into cells $\mathcal{C}_i$, and each cell can store at most one solution.
Offspring solutions are considered for addition to the cell that corresponds to their feature.
If the cell is empty, then the offspring solution is added. 
Otherwise, the new solution is only added to the archive if it has higher fitness than the existing occupant, in which case it replaces the occupant.
The aim of \mapelites algorithms is to fill the archive's cells with solutions that are as high-performing as possible.
Thus, the performance of these algorithms is often assessed by the \qdscore which is the sum of fitnesses of all solutions in the grid.
More formally, given an archive that has been tessellated into $k$ cells, the \qdscore \cite{qdunifying} is given by:
\begin{equation}
    \max_{\theta\in\Theta} \sum_{i=1}^{k} f(\theta_i), \,\,\text{where} \,\,\forall i, \Phi(\theta_i)\in \mathcal{C}_i
\end{equation}

In contrast to structured archives, unstructured archives \cite{nslc, qdunifying, aurora-original, aurora-csc} do not discretize the feature space into predefined cells.
Instead, they maintain a set of solutions and rely on measures of similarity or distance between feature vectors to preserve diversity.
Solutions are considered for addition to the archive based on a threshold distance $l$ that controls how similar solutions can be to one another. 
A new solution is added if its feature is at least $l$ away from the features of all existing solutions in the archive, ensuring novelty.
Alternatively, if the new solution lies within a distance $l$ of existing solutions, it can still be added if it has higher fitness than those nearby solutions, in which case it replaces the less effective ones.
This mechanism ensures that the archive maintains both diversity and high-quality solutions.
The $l$-value plays a crucial role in balancing these factors. 
If the $l$-value is too large, it is harder to add solutions to the archive, as they must either be significantly novel or have a very high fitness to be added.
In this case, the algorithm may struggle to find solutions which it can add to the archive, and hence may struggle to explore the feature space effectively.
Conversely, an $l$-value which is too small means that new solutions do not have to compete with existing ones, allowing poor-performing solutions to be added to the archive.
In this case, the archive will quickly populate with sub-optimal solutions, reducing the algorithm's selection pressure and thus performance.
Tuning $l$ is thus critical for achieving a balance between exploration and exploitation. 
While this approach adds a layer of complexity compared to \mapelites-based algorithms, it  allows unstructured archives to handle complex, or irregular feature spaces more flexibly than structured archives.

To address the limitations of predefined grids, some works improve grid-based methods to better handle unknown feature spaces.
For example, MAP-Elites with sliding Boundaries \cite{mesb} periodically remaps the boundaries of \mapelites grid cells to reflect the distribution of solutions, progressively increasing its resolution over time.
Alternatively, Cluster-Elites \cite{clusterelites} continuously samples the feature space to relocate the centroids of a CVT MAP-Elites grid \cite{cvt} in order to maximise coverage.
Despite these innovations, both methods remain grid-based, which are not suitable for tasks that require learned feature spaces (see \Cref{sec:unsupervised_results}).

\subsection{Multi-Objective Quality-Diversity}

Quality-Diversity algorithms have proven to be effective at finding diverse collections of solutions that optimize a single objective.
However, in many real-world scenarios, solutions must balance multiple objectives, necessitating a framework that combines diversity with multi-objective optimisation.
For example, in materials design, researchers may seek solutions that vary in structural properties (e.g. their band gaps) while exploring trade-offs between objectives such as magnetism and stability \cite{moqdcsp}.

Multi-Objective Quality-Diversity (\moqd) addresses this need by combining \mo and \qd.
The goal of \moqd is to identify solutions that are diverse in their features $\Phi: \Theta \to \mathbb{R}^d$, while for each distinct feature, capturing a Pareto Front of solutions that optimize multiple objectives $\mathbf{F}(\theta) = [f_1(\theta), f_2(\theta), \dots, f_m(\theta)]$. 

Multi-Objective Quality-Diversity was first achieved via the Multi-Objective MAP-Elites (\mome) algorithm \cite{mome}, which builds upon the \mapelites framework.
In \mome, the feature space is discretised into cells $\mathcal{C}_i$.
However, in \mome, each cell stores a Pareto front of solutions rather than a single best solution.
At each generation, offspring solutions are generated via genetic variation and then evaluated to find their fitness on each objective and  their features.
Then, when a new solution is assigned to a cell, it is added if it is belongs to the Pareto Front within the cell and replaces any solutions that it dominates.
Otherwise the new solution is discarded. 

The performance of \moqd algorithms is often measured via the \moqd score, which aggregates the quality of trade-offs and diversity across the feature space.
Specifically, the \moqd score \cite{mome} is the sum of hypervolumes of the Pareto fronts across all occupied cells:
\vspace{-2mm}

\begin{equation}
    \max_{\theta\in\Theta} \sum_{i=1}^{k} \Xi(\mathcal{P}_i), \,\,\text{where} \,\,\forall i, \mathcal{P}_i = \mathcal{P}({\theta|\Phi(\theta)\in \mathcal{C}_i})
    \label{eqn:moqd}
\end{equation}

Since \mome, several other \moqd works have been proposed \cite{mome-pgx, mome-p2c, c-mome} to improve its performance.
For example, some works leverage Reinforcement Learning to improve the performance of \mome in high-dimensional search spaces \cite{mome-pgx, mome-p2c}.
Alternatively, other methods have proposed using counterfactual agents in order to improve the coverage of behavioural niches \cite{c-mome}.
However all \moqd methods to date rely on a tessellated feature space as illustrated in \Cref{fig:teaser} in order to maintain Pareto fronts that span diverse features.
In this paper, we extend Multi-Objective Quality-Diversity to unstructured archives, which is non-trivial due to the challenges of defining Pareto fronts in continuous feature spaces.


\begin{figure*}[ht!]
    \centering
    \includegraphics[width=0.9\linewidth]{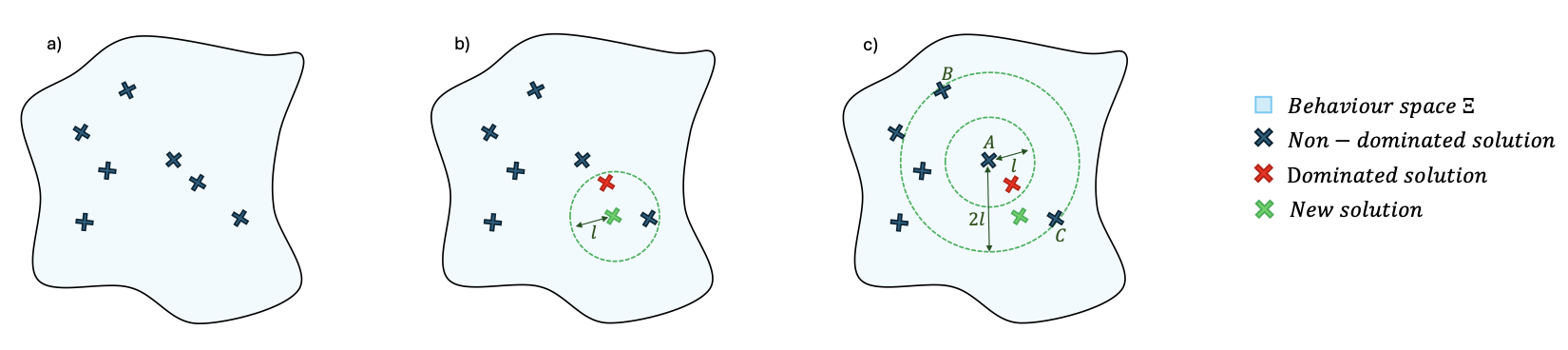}
    \caption{ a) In \mourqd solutions are stored in a continuous, unstructured manner. b) When adding a solution to the archive, it is compared with other solutions that lie within a radius $l$ in the feature space. Solutions that offer a different trade-off are kept, but solutions that are dominated by the new solution are removed. c) Removing dominated solutions can negatively impact the set of possible trade-offs within $l$ of nearby solutions, but it guarantees an improved set of trade-offs within $2l$ of them.}
    \label{fig:method}
\end{figure*}

\subsection{Unsupervised Quality-Diversity}

A key limitation of traditional Quality-Diversity algorithms is the need to define the features in advance.
In many tasks, this can be labour-intensive or very difficult, particularly when the defining features of a solution are not well understood beforehand.
To address this, \aurora introduced unsupervised Quality-Diversity \cite{aurora-original, aurora-csc, ruda}, which removes the need for a hand-defined feature.

In \aurora, the manually defined feature function $\Phi: \Theta \to \mathbb{R}^d$ is replaced by a learned one $\Phi_\text{learned}$.
Usually, $\Phi_\text{learned}$ is a dimensionality reduction technique that projects data generated by solutions during evaluation into a low-dimensional representation.
For instance, in robotics, \aurora can use an auto-encoder to encode the trajectory of a robot into a concise feature representation that highlights key aspects of its movement.
In this example, if $x$ denotes the data from the evaluation of a solution $\theta \in \Theta$, the descriptor function $\Phi_\text{learned}$ is defined as the encoding produced by an auto-encoder: $\Phi_\text{learned}(\theta) = \text{Encoder}(x)$.
By learning directly from data collected during the evaluation, the salient aspects of a solution are still captured but features do not need to be defined in advance.

Aside from learning the feature space, much of the remaining algorithmic flow of \aurora adheres to a standard \qd procedure.
In each iteration, solutions are selected, undergo variation and are added back to the archive (which is usually unstructured) based on their learned features.
However, \aurora introduces some techniques in order to improve the stability and performance of the algorithm \cite{aurora-original, aurora-csc}. 
For example, since solutions generated early in the algorithm tend to be less diverse than those generated later, the descriptor function $\Phi_\text{learned}$ is continuously retrained throughout the execution of the algorithm.
This ensures that the learned feature space adapts as the archive becomes more diverse.
On the other hand, continuous retraining can cause the features of existing solutions in the archive to change.
To address this, when retraining occurs, all of the solutions are removed from the archive and are re-added based on their updated descriptors.
Since retraining may change the shape of the latent space learned by $\Phi_\text{learned}$, when the update occurs, \aurora also dynamically updates the distance thresholds used to evaluate diversity.
For example, the Container-Size Control (CSC) approach updates the $l$-value based on the number of solutions in the archive:

\begin{equation}
    l \leftarrow l \times \big( 1 + k*(|\mathcal{A}| - N_{target})\big),
\end{equation}\label{eqn:csc}

where $|\mathcal{A}|$ is the current number of solutions in the archive $\mathcal{A}$, $N_{\text{target}}$ is the target number of solutions in the archive and $k$ is a constant that controls the rate at which the $l$-value is updated.
For a comprehensive explanation of the \aurora algorithm, readers are encouraged to refer to the original work \cite{aurora-original, aurora-csc}.


\section{Method}

In this section, we introduce \mourqdlong (\mourqd), a novel  Multi-Objective Quality-Diversity (\moqd) algorithm to unstructured and unbounded behaviour spaces, as visualised in \Cref{fig:teaser}.
Unlike traditional \moqd approaches that rely on structured archives, our method stores solutions in a fully unstructured manner throughout the feature space.
This method is particularly important in \moqd methods where the features of interest may not be known beforehand and must be learned during the algorithm’s execution \cite{aurora-original, aurora-csc, ruda}.
In such cases, the shape and bounds of the latent space are typically unknown, making it impossible to define a suitable grid structure required by existing \moqd methods.
Furthermore, even when approximate bounds of the latent space can be estimated, our results show that using such estimates for a grid-based method often leads to poor performance (see \Cref{sec:unsupervised_results}).

Even in cases where the features are hand-defined, the limits of diversity within the feature space may still be unclear.
As we demonstrate in \Cref{sec:bounded_results}, if the bounds of the \mapelites grid are defined incorrectly in these scenarios, the archive may fail to adequately represent trade-offs or capture the full spectrum of diverse solutions, resulting in sub-optimal performance.
As demonstrated in our results, addressing these challenges is critical for achieving robust and flexible \moqd in a variety of domains.

\subsection{Multi-Objective Quality-Diversity Unstructured Archive}

In \mourqd, as illustrated in \Cref{fig:method}a, we store solutions in a continuous, unstructured archive.
Storing solutions in an unstructured manner while preserving diversity in the feature space is non-trivial.
In single-objective unstructured archives, solutions are typically added to the repertoire based on a threshold radius  $l$: a new solution is added if no other solutions exist within  $l$  in the descriptor space or if it outperforms existing ones \cite{aurora-original, aurora-csc, nslc}.
This simple approach ensures that only the best-performing solutions are maintained in any local neighbourhood of the feature space.

Extending this approach to the multi-objective setting, however, introduces challenges.
In single-objective tasks, comparing solutions is straightforward due to a scalar fitness value.
By contrast, as explained in \Cref{sec:backgroundmoo}, \mo problems rely on the notion of Pareto optimality, where solutions must be evaluated across multiple objectives to determine if one dominates another.
This multi-dimensional comparison makes it difficult to decide if a solution should be added.

One possible strategy is to add a solution if it belongs to the local Pareto front within radius  $l$  and to remove any dominated solutions, as shown in \Cref{fig:method}b.
While this approach appears straightforward, it introduces a new issue.
Due to the continuous nature of the feature-space, removing dominated solutions can degrade the local Pareto front for other nearby solutions.
For example, consider solution  $A$, as illustrated in \Cref{fig:method}c.
Adding the new solution and removing a dominated one may lead to worse trade-offs within the radius  $l$  of Solution A.

Our key insight is that this trade-off degradation is mitigated when considering the collective performance over a larger region of the behaviour space.
Specifically, while local Pareto trade-offs within  a radius $l$ of Solution A  may worsen, the trade-offs within a radius of  $2l$  improve due to the inclusion of a higher-performing solution,  as illustrated in \Cref{fig:method}c.
Therefore, by selecting an appropriate value for  $l$  that results in the desired number of trade-offs within every  $2l$  radius of the feature space, the archive can maintain a balance between diversity and performance while ensuring meaningful Pareto fronts across the space.
A formal proof of this guarantee is provided in \Cref{app:proof}.

This proof guarantees an improvement in the multi-objective performance for the collection of solutions but does not explicitly guarantee diversity.
However, in practice, we observe that diversity can be achieved by tuning the $l$-value.
Moreover, in tasks where the feature space must be learned, the $l$-value can be automatically adjusted via the Container Size Control method (see \Cref{eqn:csc}), akin to single-objective unstructured repertoires.
We also note that, solutions within $2l$ of a given solution may not form a Pareto Front and may include some dominated solutions.
This is because some solutions (e.g. solutions $B$ and $C$ in \Cref{fig:method}c) which are not within $l$ from each-other have not competed via Pareto addition rules.
However, since by definition these solutions are at least $l$ away from each-other, they occupy distinct regions of the behaviour space and thereby preserving the overall diversity of the archive.

This simple yet powerful approach enables \moqd algorithms to be extended to latent spaces and unstructured tasks.
We demonstrate its effectiveness across various settings in \Cref{sec:results}.

\begin{table*}[ht!]
\centering
  \caption{Summary of evaluation tasks.}
  \label{tab:tasks}
  \scalebox{0.65}{
  \begin{tabular}{ c | p{4cm} | p{4cm} | p{4cm} | p{4cm} | p{4cm}}
  

    & \makecell{\includegraphics[width = 0.07\textwidth]{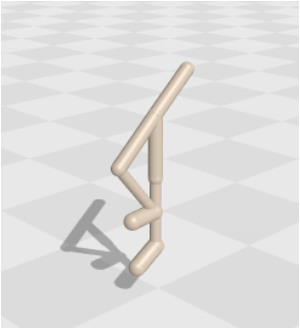}}
    
    & \makecell{\includegraphics[width = 0.07\textwidth]{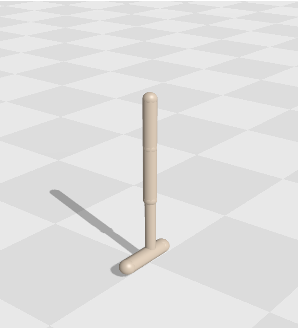}}

    & \makecell{\includegraphics[width = 0.07\textwidth]{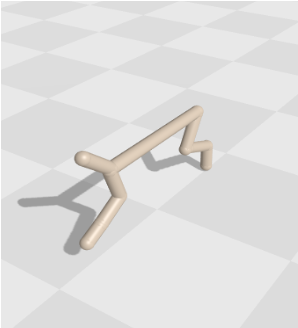}}
    
    & \makecell{\includegraphics[width = 0.07\textwidth]{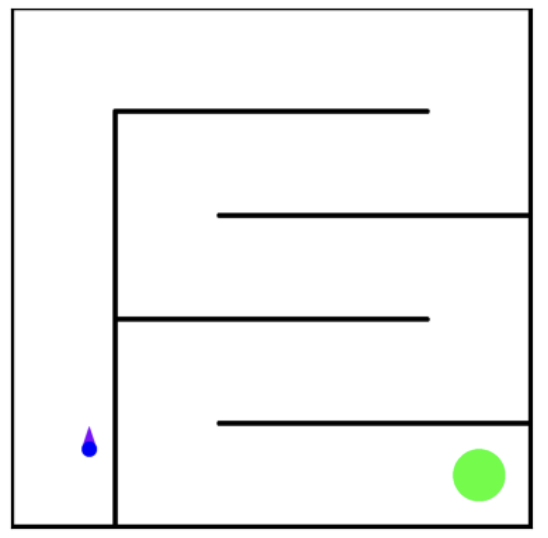}}
    
    & \makecell{\includegraphics[width = 0.07\textwidth]{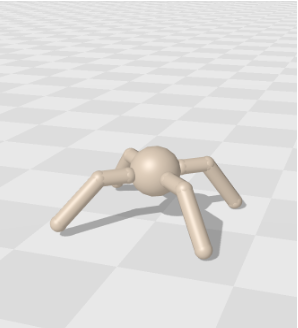}} \\

    \addlinespace[0.05cm]
    \midrule
    \addlinespace

    \textsc{Name}
    & \makecell{\walkertwo}
    & \makecell{\hopperthree}
    & \makecell{\halfcheetah}
    & \makecell{\kheperaxmulti}
    & \makecell{\antmulti}
    \\

    \addlinespace[0.05cm]
    \midrule
    \addlinespace[0.05cm]
    
    \textsc{Characteristics}
    &\makecell{-} 
    & \makecell{Tri-objective} 
    & \makecell{Unsupervised} 
    & \makecell{Unsupervised,\\Deceptive} 
    & \makecell{Unbounded} 
    \\

    \addlinespace[0.05cm]
    \midrule
    \addlinespace[0.05cm]

    \textsc{Feature}
    &\makecell{Feet Contact} 
    & \makecell{Feet Contact} 
    & \makecell{-} 
    & \makecell{-} 
    & \makecell{Final $x$-$y$ location} 
    \\

    \addlinespace[0.05cm]
    \midrule
    \addlinespace

    \makecell{\textsc{Data for} \\ \textsc{Unsupervised Features}}
    & \makecell{-} 
    & \makecell{-} 
    & \makecell{State trajectory\\ vector} 
    & \makecell{Image of final\\ maze state} 
    & \makecell{-} 
    \\

    \addlinespace[0.05cm]
    \midrule
    \addlinespace
    
    \textsc{Objectives}&

    \makecell{Forward velocity,\\ Energy consumption}
    
    &
    \makecell{Forward velocity,\\ Energy consumption, \\Jumping height}

    &
    \makecell{Forward velocity,\\ Energy consumption}
    
    &
    \makecell{Energy consumption, \\ Distance to Goal}
    
    &
    \makecell{Survival bonus - Energy, \\ Reproducibility}
    \\
    
    \addlinespace[0.05cm]
    \midrule
    \addlinespace
    
    \textsc{Comparative Baselines}&

    \makecell{\mome}
    
    &
    \makecell{\mome}

    &
    \makecell{\moauroragrid}
    &
    \makecell{\moauroragrid}

    &
    \makecell{\mome,\\ \momelarge, \\ \momesmall}
 

  \end{tabular}}
  \label{fig:experiments}
\end{table*}

\section{Experimental Setup}

\subsection{Evaluation Tasks}

We evaluate our methods on $5$ different continuous control robotics tasks which are summarised in \Cref{tab:tasks}.
In each of the tasks, the solutions correspond to neural-network controllers that control the robot at each time step $t$ based on their current state $s$.
We further categorise the tasks into three groups: traditional tasks, unsupervised tasks and unbounded tasks.

\subsubsection{Traditional Multi-Objective Quality-Diversity Tasks}
To ensure that our method can match the performance of \moqd methods that rely on a grid, we first evaluate our approach against two traditional \moqd robotics control tasks using the Brax suite \cite{brax}.
The first task, \walkertwo, is a bi-objective task where the aim is to find a controller that maximises the velocity and minimises the control cost of the Walker robot \cite{mome, mome-pgx, mome-p2c}.
The second task, \hopperthree, is a tri-objective task which uses the Hopper robot morphology and the same objectives as \walkertwo but additionally uses jumping height as a third objective.
In both tasks, the features of the robot are categorised by the proportion of time that the robot spends on each of its feet \cite{mome, mome-pgx, mome-p2c}.

\subsubsection{Unsupervised Multi-Objective Quality-Diversity Tasks}

One of the key challenges that our method is trying to tackle is to be able to apply \moqd methods to unsupervised settings.
Therefore, we also run our algorithm on two unsupervised \moqd tasks. 
In the first task, \halfcheetah, the HalfCheetah robot \cite{brax} must maximise its forward velocity while minimising its energy consumption.
In this task the features are learned using an LSTM from the concatenation of the robot's state, including its feet contact information, at each time step. 
The second unsupervised task, \kheperaxmulti, involves a roomba-style robot navigating a maze. 
We use the snake maze from the Kheperax suite \cite{kheperax} which we adapt to become a multi-objective task.
In particular, the first objective in the task is the energy consumption of the robot and the second is the total of the negative distance to the goal at each time step. 
Importantly, the layout of the maze requires the robot to initially move farther away from the goal before it can approach it, making the second objective highly deceptive.
To encourage goal-reaching behaviour, we add a small positive bonus to the reward when the robot is close to the goal.
In this task, we use a $64 \times64$ image of the robot in the maze at the final time-step $t$ to learn features via a CNN auto-encoder.
These images are normalised, so only the robot’s position is encoded, and any information about the maze is removed.
We include parameter information for both the auto-encoders and \aurora-based parameters for these tasks in \Cref{app:aurora_hyperparams}.

\subsubsection{Unbounded Multi-Objective Quality-Diversity Tasks}

The final task we assess our method on is \antmulti.
In this task, the objective is for the Ant robot \cite{brax} to travel in any direction for $1000$ time-steps, with the feature being its final $x$-$y$ location.
The challenge lies in not knowing a-priori how far the Ant can travel within this time-frame, making it an “unbounded” task.
Typically, determining the correct bounds for the feature space requires empirically identifying the robot’s potential limits. 
However, in this task, we assume these bounds are unknown and demonstrate that \mourqd can automatically adapt by discovering the effective limits of the feature space.
The objectives in this task are 1) a survival bonus - energy consumption and 2) the reproducibility of the solution \cite{performance-reproducibility}. 

\bigbreak

Further information for the tasks can be found in \Cref{app:rewardfns}.

\begin{figure*}
    \centering
    \includegraphics[width=0.95\linewidth]{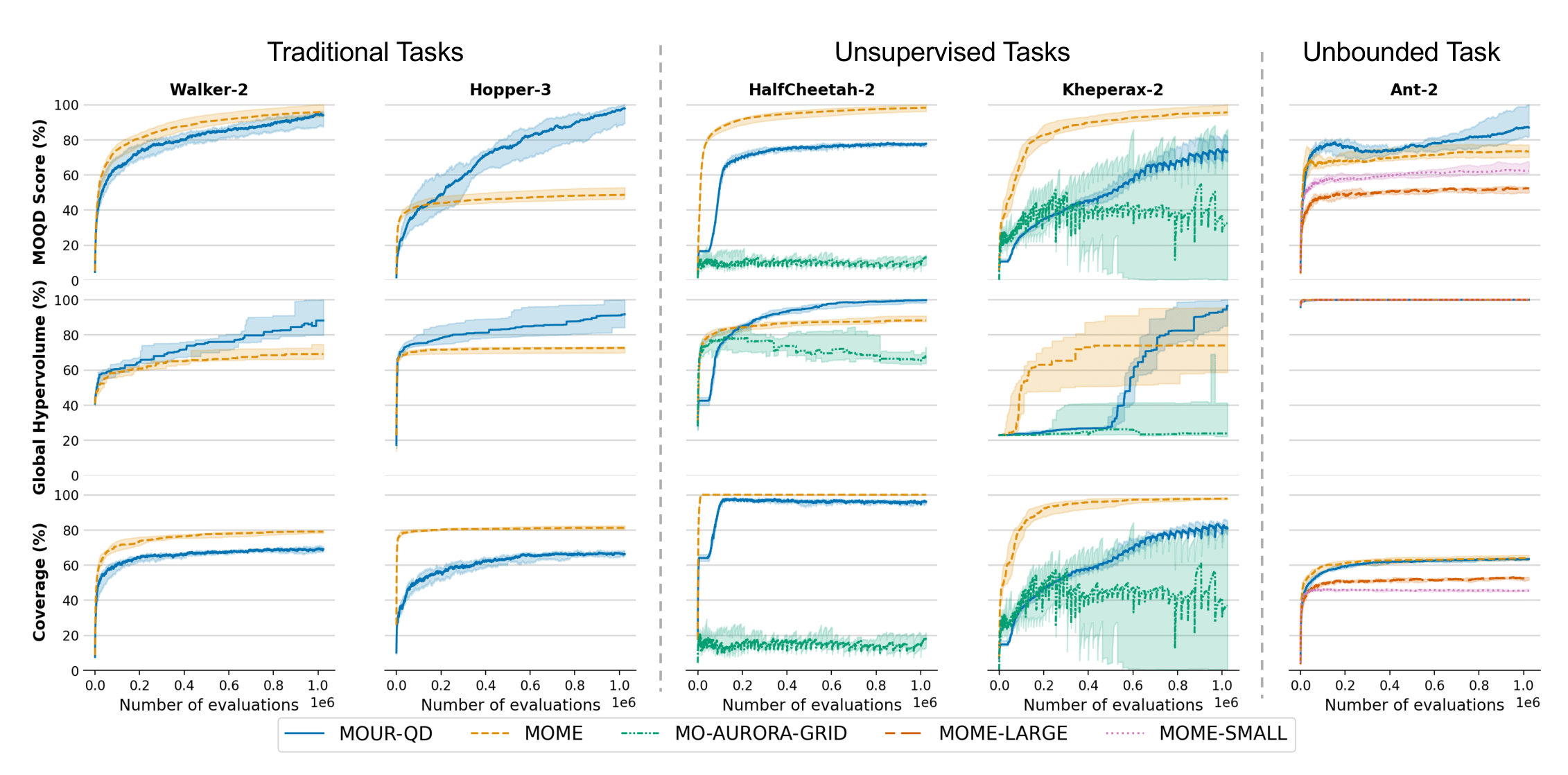}
    \caption{Performance of \mourqd compared to all other baselines. The line shows the median score and the shaded region shows the interquartile range across \replications seeds. In unsupervised tasks, \mome provides an approximate upper bound for performance, rather than a comparative baseline.}
    \label{fig:results}
\end{figure*}

\subsection{Baselines} 

\subsubsection{Traditional Multi-Objective Quality-Diversity Tasks}
In both traditional \moqd tasks we compare \mourqd to \mome.
We exclude other \moqd methods that incorporate additional components, such as gradient-based updates or counterfactual mechanisms \cite{c-mome, mome-p2c, mome-pgx}, as these rely on components that are beyond the scope of this work.
Instead, we note that these techniques could complement \mourqd and this is left as a direction for future work.

\subsubsection{Unsupervised Multi-Objective Quality-Diversity Tasks}
In both unsupervised \moqd tasks, we compare \mourqd to two baselines.
The first baseline is \mome which uses the ground-truth features.
Since this baseline leverages pre-defined features, we consider it to be an upper bound for performance.
Secondly, we compare our method to a \mo version of \aurora that uses a grid (\moauroragrid).
In this baseline, learned features $\phi_{\text{learned}}$ are first extracted by generating random solutions and training the auto-encoder.
These learned features are then used to estimate the bounds of a \mome grid, which is tessellated in the usual manner.
Once the grid is defined, it remains fixed for the duration of the algorithm, and the process proceeds similarly to \mome but using $\phi_{\text{learned}}$ for features.
Unlike \mourqd, \moauroragrid does not use an unstructured repertoire, and hence it does not employ container-size control.
However, the auto-encoders are periodically retrained at the same frequency to ensure consistent learning of features.

\subsubsection{Unbounded Multi-Objective Quality-Diversity Tasks}
In the unbounded task, we evaluate our method, \mourqd, against three \mome baselines to assess how \mome performance is affected when grid boundaries are misestimated.
To establish a reference for the true limits of the feature space, we first run \mourqd and \mome using a very large grid over several seeds.
From these experiments, we estimate the true bounds of the feature space to be approximately $[-60, 60]$.
Using this information, we define three \mome baselines.
The first baseline, \mome, uses the correctly estimated bounds of $[-60, 60]$, representing an ideal setup.
The second baseline, which we denote as \momesmall, intentionally restricts the grid bounds to $[-20, 20]$ to simulate an underestimation of the feature space limits.
Conversely, the final baseline, \momelarge, uses overestimated bounds of $[-80, 80]$, introducing excess space in the grid.
By comparing these baselines, we aim to highlight how incorrect boundary definitions affect the performance of \mome and demonstrate the advantages of \mourqd in unbounded environments.

\subsection{Experiment Designs}

All algorithms were run with a batch size of $256$ for $4000$ iterations, meaning a total of $1,024,000$ evaluations.
We use the isoline variation operator \cite{isoline} with $\sigma_1=0.005$, and $\sigma_2=0.05$ across all experiments.
For all grid-based methods, we use a CVT tessellation \cite{cvt} to discretise the feature space into $512$ cells, each with a maximum Pareto Front size of $10$.
Accordingly, for \mourqd we used a maximum archive size of $512 \times 10 = 5120$, to ensure the same maximum population size.
We provide further information about the parameters for unsupervised tasks in \Cref{app:aurora_hyperparams} and the $l$-values used in each environment in \Cref{app:l_values}.

\subsection{Evaluation Metrics}

Unstructured archives present challenges for evaluation using standard metrics \cite{aurora-original, aurora-csc}.
Unlike structured archives, where metrics like the sum of solution fitnesses in the grid (\qdscore) naturally incorporate both quality and diversity, unstructured archives lack this inherent structure to reflect diversity.
For this reason, to evaluate \mourqd, we project the solutions onto a \mome grid and use standard \moqd metrics as a proxy for performance.
In unsupervised tasks, we use the "true" features of solutions in the archive for projection.
Similarly, since the \moauroragrid, \momesmall and \momelarge baselines use different tessellations, we project solutions from these onto a \mome repertoire and evaluate performance accordingly.
This allows for a fair comparison while still capturing both quality and diversity.
Using these projected repertoires, we evaluate \mourqd and baseline algorithms using three metrics:

\begin{enumerate}[leftmargin=*]
    \item \textbf{\moqdscore}: The \moqdscore quantifies the performance of solutions in the archive by summing the hypervolumes of the Pareto fronts within each cell of the \mome grid (see \Cref{eqn:moqd}). This metric, similar to the \qdscore in single-objective settings \cite{qdunifying}, captures both the diversity of solutions across the feature space and their performance.
    \item \textbf{\globalhypscore}: The global hypervolume measures the hypervolume of the Pareto front formed from all solutions in the archive. This metric evaluates the best trade-offs across objectives found by an algorithm, independent of the feature space. As well as presenting this numerically, we provide visualisations of the Global Pareto Fronts in \Cref{app:global_pfs}.
    \item \textbf{\coverage}. The \coverage is the fraction of cells in the archive that contain at least one solution. This metric reflects the diversity of solutions within the feature space.
\end{enumerate}

The reference points use to calculate the \moqdscore and the \globalhypscore for each task can be found in \Cref{app:refpoints}.
\section{Results}\label{sec:results}

\begin{figure*}[t] 
    \centering
    \begin{subfigure}[b]{0.49\textwidth} 
        \centering
        \includegraphics[width=\linewidth]{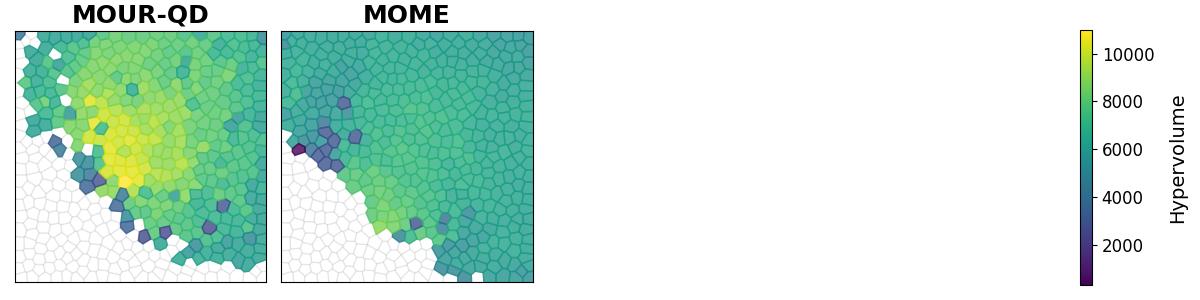}
        \caption{\walkertwo}
        \label{fig:rep_walker}
    \end{subfigure}
    \hfill
    \begin{subfigure}[b]{0.49\textwidth}
        \centering
        \includegraphics[width=\linewidth]{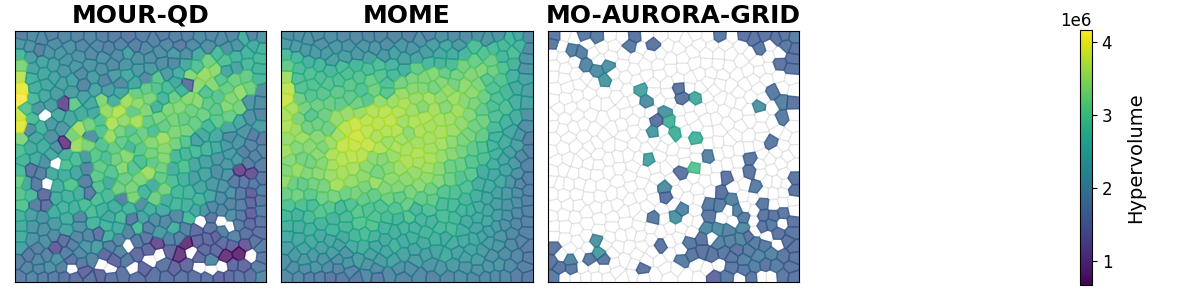}
        \caption{\halfcheetah}
        \label{fig:rep_hc}
    \end{subfigure}

    \begin{subfigure}[b]{0.49\textwidth}
        \centering
        \includegraphics[width=\linewidth]{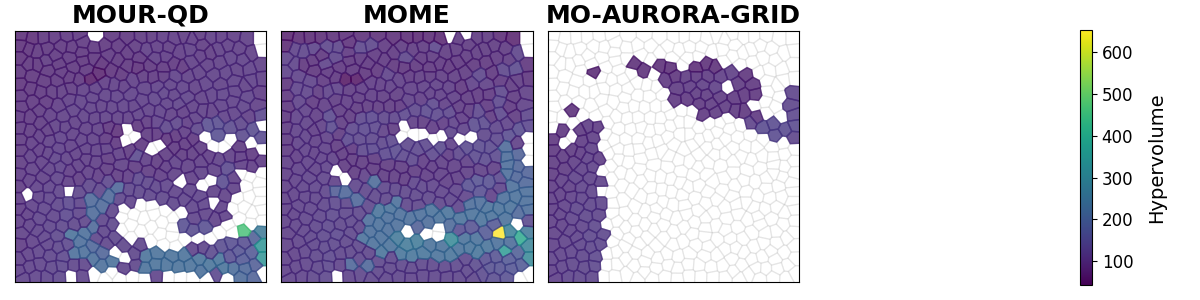}
        \caption{\kheperaxmulti}
        \label{fig:rep_khep}
    \end{subfigure}
    \hfill
    \begin{subfigure}[b]{0.49\textwidth}
        \centering
        \includegraphics[width=\linewidth]{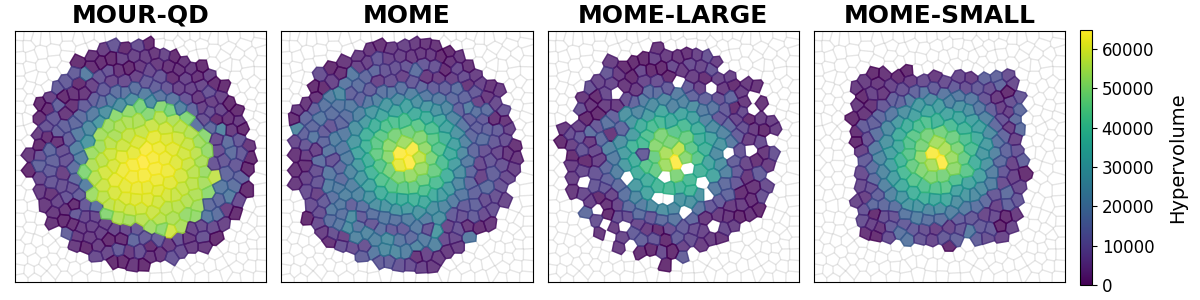}
        \caption{\antmulti}
        \label{fig:rep_ant}
    \end{subfigure}

    \caption{Final repertoire plots from the median run of each algorithm. Only tasks with 2-dimensional features are shown.}
    \label{fig:repertoires}
\end{figure*}

The results of our experiments are presented in \Cref{fig:results}.
All of our baselines were run for \replications replications, and we report $p$-values from a Wilcoxon signed-rank test \cite{wilcoxon1992individual} using a Holm-Bonferroni correction \cite{holm_bonf}.
Overall, our results demonstrate that \mourqd matches or outperforms all baselines in all environments.

\subsubsection{Traditional Multi-Objective Quality-Diversity Tasks}\label{sec:moqd_results}

\Cref{fig:results} demonstrates that is competitive with \mome in both traditional \moqd tasks.
When comparing \mourqd and \mome on the \moqdscore, there is no significant difference between the two methods in \walkertwo ($p=0.232$), while \mourqd significantly outperforms \mome in \hopperthree ($p<0.004$), approximately doubling the \moqdscore.
This result indicates that \mourqd effectively discovers a wide range of high-performing trade-offs that are diverse across features, which is the central goal of \moqd.

Additionally, \mourqd outperforms \mome on \textsc{global-} \textsc{hypervolume} in both environments (p<0.006), showing its ability to identify optimal trade-offs when features are not explicitly considered.
We hypothesise that this is due to the continuous nature of \mourqd's repertoire, which promotes greater competition among solutions.
Unlike cell-based approaches, where solutions near the boundaries of a cell do not compete with those in adjacent cells, the unstructured design of \mourqd ensures that all solutions compete directly, regardless of their position in the feature space.
This more competitive addition process likely contributes to its improved overall performance.
This advantage is further visualised in \Cref{fig:rep_walker}, where many cells exhibit high hypervolumes.

A limitation of \mourqd is that it achieves a lower coverage score than \mome in both environments ($p<0.002$) as shown in \Cref{fig:results}.
Examining \Cref{fig:rep_walker} shows that some cells in the \mourqd repertoire are empty.
These empty may partly result from projecting the continuous \mourqd repertoire onto a \mome grid.
However, other empty cells in \Cref{fig:rep_walker} suggest that \mourqd may struggle to retain solutions near the boundaries of the feature space, indicating that \mourqd is susceptible to erosion in these areas \cite{qdunifying}.

Nevertheless, our results demonstrate that \mourqd performs comparably to \mome and even outperforms it in certain domains, despite using an unstructured archive.

\subsubsection{Unsupervised Multi-Objective Quality-Diversity Tasks}\label{sec:unsupervised_results}

\Cref{fig:results} demonstrates that \mourqd also performs well in unsupervised environments, achieving approximately $80\%$ of the \moqdscore of \mome in both environments.
Moreover, \mourqd achieves a statistically significant higher \globalhypscore in \halfcheetah ($p<0.006$).
We emphasise that, in these tasks, \mome should not considered as a comparative baseline.
Specifically, \mome benefits from hand-defined features which were constructed for evaluation in this task but would not ordinarily be available in unsupervised tasks.
Additionally, \mome is evaluated directly on the repertoire it is evolved on which is defined by these hand-defined features.
By contrast, \mourqd may lose solutions when its continuous repertoire is projected onto the discrete grid used to compute metrics.
Moreover, although the tasks were designed such that the data used to learn the features is highly correlated with the hand-defined features of the grid used to compute the metrics, they may not align perfectly with the features learned by the auto-encoders.
For example, in \halfcheetah, \mourqd may find diversity in other aspects of the robot's trajectory, rather than just its feet contact.
Thus, it is particularly notable that \mourqd remains competitive to \mome — despite the latter effectively representing an upper bound— showing that it has been able to discover a large variety of solutions.

Notably, \moauroragrid performs poorly in \halfcheetah and is highly unstable in \kheperaxmulti, likely due to shifting latent space bounds as the auto-encoder trains.
As these bounds change, the grid fails to accurately capture the evolving feature space, resulting in many cells being left empty, as illustrated in \Cref{fig:rep_hc} and \Cref{fig:rep_khep}.
By contrast, \mourqd is able to adapt as the encoder trains and achieves a fuller repertoire, confirming its ability to effectively handle unsupervised tasks.
These results highlight the potential for applying \mourqd requiring unsupervised feature discovery.

\subsubsection{Unbounded Multi-Objective Quality-Diversity Tasks}\label{sec:bounded_results}

\Cref{fig:results} also demonstrates that \mourqd performs effectively in tasks where the bounds of the feature space are unknown.
Notably, \Cref{fig:results} demonstrates that when the boundaries of the \mome grid are set either too large or too small, performance is suboptimal, both in terms of \moqdscore and \coverage. 
This is further illustrated in the repertoire plots in \Cref{fig:rep_ant}.
Specifically, in \momelarge, when the bounds of the \mome grid are too large, the algorithm explores the feature space effectively, reaching its boundaries, but fails to find high-performing trade-offs.
Conversely when the bounds are too small, as in \momesmall, the algorithm shows improved exploitative performance, as indicated by the higher hypervolume in \Cref{fig:rep_ant}, but it is unable to fully explore the feature space.
These findings highlight the critical importance of carefully setting the correct grid boundaries in \mome, as inaccurate bounds can hinder performance.
When \mome is set with the correct feature space boundaries, it performs reasonably well.
However, we find that \mourqd is able to achieve a higher \moqd performance ($p<0.03$) than \mome, even without the need to specify these boundaries in advance.
As visualised \Cref{fig:rep_ant}, this is because \mourqd is able to both explore effectively by automatically adapting to the feature space limits, while also finding better trade-offs compared to \mome.

\section{Conclusion}

In this work, we presented \mourqdlong (\mourqd), a novel \moqd method designed to address challenges in tasks with unstructured archives and unknown feature space bounds. We demonstrated that \mourqd matches or outperforms existing \moqd methods on traditional tasks, while also excelling in scenarios where feature space boundaries are unknown or poorly defined.
Moreover, we showed that \mourqd can be used in tasks that require learned feature spaces, opening up \moqd to unsupervised domains.
In the future, we aim to improve the performance of the coverage of our algorithm by addressing its susceptibility to corrosion \cite{qdunifying} and apply our method to tackle more complex domains, such as protein design or latent exploration tasks.

\section*{Acknowledgements}
This work was supported by PhD scholarship funding for Hannah from InstaDeep.


\bibliographystyle{ACM-Reference-Format}
\bibliography{main}

\clearpage
\appendix
\onecolumn
\numberwithin{equation}{section}

\section*{Supplementary Materials}

\section{Proof of Local Improvement}\label{app:proof}

\begin{theorem}
    Consider an archive $\mathcal{A}$ which comprises $n$ solutions $x_1, ..., x_n \in \mathcal{X}$ and a distance metric $d: \mathcal{X} \rightarrow \mathbb{R}$. Let $\mathcal{P}(x, r, \mathcal{A})$ be defined as the set of solutions in the archive $\mathcal{A}$ within radius $r$ of solution $x$:
    \begin{equation}
        \mathcal{P}(x, r, \mathcal{A}) := \{y \in \mathcal{A}\,|\, d(x, y) < r\}\,.
    \end{equation}
    Now consider $x_m \notin \mathcal{A}$ such that $x_m \succ x_i$ and $d(x_i, x_m)<r$ for an arbitrary solution $x_i\in \mathcal{A}$.
    Let $\mathcal{A}'$ denote an alternative archive such that $\mathcal{A}' := \mathcal{A}\cup\{x_m\}\backslash\{x_i\}$. Then $ \forall x_j \in \mathcal{P}(x_i, r, \mathcal{A})$
        \begin{equation}
            \Xi(\mathcal{P}(x_j, 2r, \mathcal{A}')) > \Xi(\mathcal{P}(x_j, 2r, \mathcal{A}))\,,\label{eqn:theorem}
        \end{equation}
    where $\Xi$ denotes the hypervolume metric .
\end{theorem}

\begin{proof}
    First, we note that:
    \begin{align}
    d(x_j, x_m) &\leq d(x_j, x_i) + d(x_i, x_m)\label{eqn:triangle} \\
        &< r + r\label{eqn:r} \\ 
        & = 2r \label{eqn:radius} \,
    \end{align}
where \Cref{eqn:triangle} holds via the triangle inequality and \Cref{eqn:r} holds since $x_j \in \mathcal{P}(x_i, r, \mathcal{A})$ and $d(x_i, x_m)<r$ is assumed. \Cref{eqn:radius} implies that $x_m \in \mathcal{P}(x_j, 2r, \mathcal{A}')$. We also note that $x_i\in \mathcal{P}(x_j, 2r, \mathcal{A})$ since $x_i\in \mathcal{P}(x_j, r, \mathcal{A})$ by definition, but $x_i\notin \mathcal{P}(x_j, 2r, \mathcal{A}')$ since $x_i\notin \mathcal{A}'$. Hence, since $\mathcal{A}$ is the same as $\mathcal{A}$ except that $x_i \in \mathcal{A}$ and $x_m \in \mathcal{A}'$, we have that

\begin{equation}
     \mathcal{P}(x_j, 2r, A') = \mathcal{P}(x_j, 2r, A')\cup\{x_m\}\backslash\{x_i\}\,.
     \label{eqn:difference}
\end{equation}
In other words, the collection of solutions with a radius of $2r$ of $x_j$ is the same in $\mathcal{A}$ as it is in  $\mathcal{A}$, except $x_i$ is replaced by $x_m$. However, since $x_m \succ x_i$, by the monotonic property of hypervolume metric, \Cref{eqn:theorem} follows.
\end{proof}

\section{Experimental Setup Details}

\subsection{Fitness Functions}\label{app:rewardfns}

In this section, we include further details about the reward functions for each of the environments.

For \walkertwo and \halfcheetah, the first objective is the forward velocity of the robot, defined as follows:

\begin{equation}
    f_1 = \sum_{t=1}^T \frac{x_t - x_{t-1}}{\delta t}\,,\\
\end{equation}\label{reward:forward}

where $x_t$ and $x_{t-1}$ denote the positions of the robot's centre of gravity at time-steps $t$ and $t-1$ respectively and $\delta t$ denotes the length of one time-step.
The second objective for these tasks is to minimise the energy consumption which defined as follows:

\begin{equation}
    f_2 = - \sum_{t=1}^T ||a_t||_2\,,\\
\end{equation}\label{reward:energy}

where $|| \cdot ||_2$  denotes the Euclidean norm and $a_t$ denotes the action taken by the robot at time-step $t$.

\hopperthree is a tri-objective environment where the first two objectives are the forward velocity of the robot and its energy consumption, taken as \Cref{reward:forward} and \Cref{reward:energy} respectively. The third objective in \hopperthree is the robot's jumping height, given as:

\begin{equation}
    f_3 = \sum_{t=1}^T z_t\,, \\
\end{equation}\label{reward:jump}

where $z_t$ denotes the height of the robot's torso at time-step $t$.

The final task from the Brax suite \cite{brax} is \antmulti.
In this task, we consider the problem of Uncertain Quality-Diversity \cite{uqd} (\uqd).
In particular, \uqd seeks to address a key limitation of single-objective \qd algorithms which lies in their elitism.
While elitism does help \qd algorithms discover high-performing solutions, it can also pose challenges in stochastic environments.
In environments that have inherent randomness, solutions can be "lucky" and achieve a high fitness in one evaluation, but may not be robust across multiple evaluations.
However, since traditional \qd algorithms store at most one solution per cell, these lucky solutions may be kept and cause the algorithm to inadvertently discard solutions that are more reliable or robust across multiple evaluations.

One approach for addressing this limitation is to include reproducibility as a meta-objective within \moqd algorithms \cite{performance-reproducibility}.
By explicitly optimizing for both fitness and reproducibility, this approach can help mitigate the limitations of elitism, ensuring that selected solutions are not only high-performing but also consistent and robust across evaluations.
Therefore, in the \antmulti environment, instead of a traditional multi-objective task where all objectives are derived solely from evaluating solutions, we use \moqd to optimise trade-offs between the fitness of solutions and their reproducibility. 
In particular, we take use the standard fitness function used in Ant-Omni tasks in single-objective Quality-Diversity environments, which combines the robots energy consumption and a survival bonus \cite{pga} as a main objective.
Then, rather than evaluate each solution just once, we evaluate them 16 times.
The first fitness of the \moqd algorithms, is the average fitness of the solution across these 16 replications: 

\begin{equation}
    f_1 = \sum_{i=1}^{16} \Tilde{f}_i(\theta)\,.
\end{equation}

Here, $\Tilde{f}_i$ denotes the fitness from the $i$-th evaluation of the solution $\theta$. The second objective in \antmulti is a solution's reproducibility, which we calculate as:

\begin{equation}
    f_2 = - \text{std}\big(\Tilde{\phi}_i(\theta), ..., \Tilde{\phi}_{16}(\theta)\big)\,,
\end{equation}

where \text{std} denotes the standard deviation, and $\Tilde{\phi}_i(\theta)$ denotes the feature from the $i$-th evaluation of the solution $\theta$.

The final environment that we evaluate our methods on is \kheperaxmulti. In this task, a roomba-style robot must navigate a deceptive maze in order to reach a target. We take the first objective in this environment to be the energy consumption of the robot, as in the traditional Kheperax environments \cite{nslc, aurora-original}. Similar to the Brax environments, this is given as:

\begin{equation}
    f_1 = - \sum_{t=1}^T ||a_t||_2\,,\\
\end{equation}

where $a_t$ denotes the wheel velocities of the robot at time-step $t$.
To make the task multi-objective, we use a time-step distance to the goal as a second objective and use a small positive bonus when the robot is within $0.12$ from the goal. This objective can be defined as:

\begin{equation}
    f_2 = \sum_{t=1}^T r_t\,,\\
\end{equation}

where,


\begin{equation}
     r_t =
    \begin{cases}
        -  || x - goal||_2\,\,\, \text{if distance > 0.12}\,, \\ 
        + \,5\,\,\,\,\,\,\,\,\,\,\,\,\,\,\,\,\,\,\,\,\,\,\,\,\,\,\,\,\,\,\, \text{otherwise}\,.
    \end{cases}
\end{equation}
Here $x$ and $goal$ denote the positions of the robot and the target respectively.
We introduced the reward bonus after experimenting with the \mome algorithm in the \kheperaxmulti environment, where we observed that the algorithm struggled to explore the maze effectively without this additional signal. Specifically, without the bonus, \mome tended to get trapped in deceptive local optima, repeatedly finding different trade-offs within the same region of the search space, instead of exploring other areas of the maze. This issue arose because, in the absence of the bonus, the algorithm could continue to improve by refining its solutions within this local region, rather than being incentivized to explore beyond it.
In contrast to single-objective algorithms, which have been shown to work effectively without such a reward signal \cite{noveltysearch, aurora-original}, we found that the multi-objective nature of this task required a slightly stronger reward signal to promote broader exploration. The bonus effectively encouraged the algorithm to diversify its search, preventing it from getting stuck in local optima and enabling more thorough exploration of the maze.

\subsection{AURORA 
Hyperparameters}\label{app:aurora_hyperparams}
In \Cref{tab:aurora_hyps} we include the parameters we used for all unsupervised tasks (\halfcheetah and \kheperaxmulti).
For all baselines which used container-size control (see \cref{eqn:csc}), we performed container-size control after every iteration and set the target size of the archive $N_\text{target}$ to be $95\%$ of the archive's maximum size, $4,864$.
In both tasks, we trained the encoders in linearly increasing intervals of $2$, i.e. at iterations $2, 4, 8, 16$ etc. In both tasks, each time the encoders were trained, they were trained for a maximum of $200$ epochs, with early stopping to prevent over fitting. 

\begin{table}[ht!]
    \centering
        \caption{Reference points}
    \begin{tabular}{  l | l | l }
        \toprule
    Parameter & \halfcheetah & \kheperaxmulti \\
    \midrule
    \midrule
    Encoder & LSTM & CNN \\
    \midrule
    Layers &  1 & \makecell{Encoder: [16, 16, 16]\\ Decoder: [16, 16, 16]} \\
    \midrule
    Hidden size & \multicolumn{2}{c}{\makecell{10}}\\
    \midrule
    Batch size & \multicolumn{2}{c}{\makecell{256}}\\
    \midrule
    Learning rate & \multicolumn{2}{c}{\makecell{0.001}}\\
    \end{tabular}      	
    \label{tab:aurora_hyps}
\end{table}

\newpage

\subsection{Unstructured Repertoire Parameters} \label{app:l_values}

\Cref{tab:hyps} shows the $l$-values and $k$ used in container size control (\Cref{eqn:csc}) for each of the environments.

\begin{table}[h!]
    \centering
        \caption{$l$-values and $k$ in each task}
    \begin{tabular}{ | l | l | l | l | l | l| }
        \toprule
         & \walkertwo & \hopperthree& \halfcheetah & \kheperaxmulti & \antmulti \\
        \midrule
        \midrule
        Initial $l$-value & 0.03 & 0.05 & 0.01 & 0.01 & 1\\
        \midrule
        $k$ & - & - & 0.0001 & 0.001 & -\\
        \bottomrule
    \end{tabular}      	
    \label{tab:hyps}
\end{table}

\subsection{Hypervolume Reference Points} \label{app:refpoints}

\Cref{tab:refpoints} shows the reference points used for each of the environments.

\begin{table}[h]
    \centering
        \caption{Reference points}
    \begin{tabular}{ | l | l | }
        \toprule
    Environment & Reference Point \\
    \midrule
    \midrule
    \walkertwo &  [-210, -15] \\
    \midrule
    \hopperthree &  [-750, -3, 0] \\
    \midrule
    \halfcheetah & [-2000, -800] \\
    \midrule
    \kheperaxmulti & [-1, -300] \\
    \midrule
    \antmulti & [-300, -50] \\
    \bottomrule
    \end{tabular}      	
    \label{tab:refpoints}
\end{table}

\section{Global Pareto Fronts}\label{app:global_pfs}
In this section we include visualisations of the global Pareto Fronts achieved by each of the algorithm in each task.

\begin{figure}[hb!]
    \centering
    \includegraphics[width=0.99\linewidth]{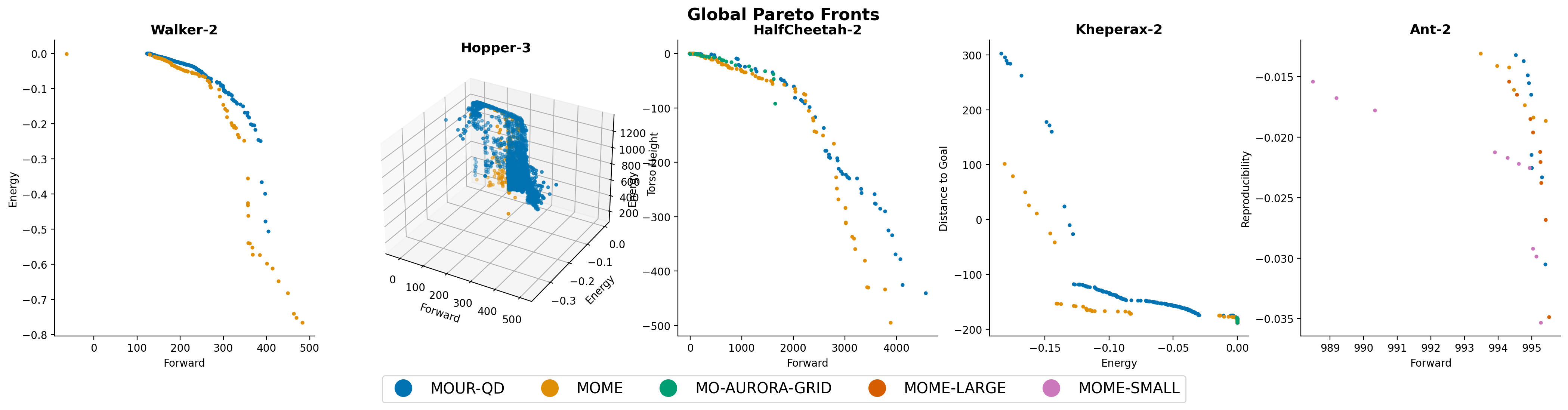}
    \caption{Global Pareto Fronts obtained by each algorithm.}
    \label{fig:global_pfs}
\end{figure}

The plot demonstrates that \mourqd achieves high-performing global Pareto Fronts, with many solutions dominating the solutions found by other algorithms.

\end{document}